\documentclass{article}

\usepackage{microtype}
\usepackage{graphicx}
\usepackage{subfigure}
\usepackage{booktabs} % for professional tables

\usepackage{hyperref}

\usepackage[accepted]{icml2019}

\usepackage{amsmath,amsfonts,amsthm} % Math packages
\usepackage{algorithm}
\usepackage{algorithmic}
\usepackage[T1]{fontenc}
\usepackage{enumitem}
\usepackage{multirow}
\usepackage{url}

\newtheorem{claim}{Claim}
\newtheorem{prop}{Proposition}
\newtheorem{defi}{Definition}
\newtheorem{cor}{Corollary}[prop]

\icmltitlerunning{Understanding and Utilizing Deep Neural Networks Trained with Noisy Labels}

\begin{document}

\twocolumn[
\icmltitle{Understanding and Utilizing Deep Neural Networks\\ Trained with Noisy Labels}

%\icmlsetsymbol{equal}{*}

\begin{icmlauthorlist}
	\icmlauthor{Pengfei Chen}{cuhk,tencent}
	\icmlauthor{Benben Liao}{tencent}
	\icmlauthor{Guangyong Chen}{tencent}
	\icmlauthor{Shengyu Zhang}{cuhk,tencent}	
\end{icmlauthorlist}

\icmlaffiliation{cuhk}{Department of Computer Science and Engineering, The Chinese University of Hong Kong}
\icmlaffiliation{tencent}{Tencent Technology}

\icmlcorrespondingauthor{Guangyong Chen}{gycchen@tencent.com}

% You may provide any keywords that you
% find helpful for describing your paper; these are used to populate
% the "keywords" metadata in the PDF but will not be shown in the document
\icmlkeywords{Machine Learning, ICML}

\vskip 0.3in
]

% this must go after the closing bracket ] following \twocolumn[ ...

% This command actually creates the footnote in the first column
% listing the affiliations and the copyright notice.
% The command takes one argument, which is text to display at the start of the footnote.
% The \icmlEqualContribution command is standard text for equal contribution.
% Remove it (just {}) if you do not need this facility.

\printAffiliationsAndNotice{}  % leave blank if no need to mention equal contribution
%\printAffiliationsAndNotice{\icmlEqualContribution} % otherwise use the standard text.

\begin{abstract}
Noisy labels are ubiquitous in real-world datasets, which poses a challenge for robustly training deep neural networks (DNNs) as DNNs usually have the high capacity to memorize the noisy labels. In this paper, we find that the test accuracy can be quantitatively characterized in terms of the noise ratio in datasets. In particular, the test accuracy is a quadratic function of the noise ratio in the case of symmetric noise, which explains the experimental findings previously published. Based on our analysis, we apply cross-validation to randomly split noisy datasets, which identifies most samples that have correct labels. Then we adopt the Co-teaching strategy which takes full advantage of the identified samples to train DNNs robustly against noisy labels. Compared with extensive state-of-the-art methods, our strategy consistently improves the generalization performance of DNNs under both synthetic and real-world training noise.

% real-world  Experiments on both synthetic and real-world noisy labels demonstrate that our strategy outperforms state-of-the-art methods in terms of generalization performance. Specifically, we verify on (i) the CIFAR-10 dataset by manually corrupting the labels with different types of noise, and (ii) the WebVision dataset, which is a large benchmark crawled from websites, containing real-world noisy labels.  
\end{abstract}

\section{Introduction}

\label{Sec_Introduction}
The remarkable success of DNNs on supervised learning tasks heavily relies on a large number of training samples with accurate labels. Correctly labeling extensive data is too costly while alternating methods such as crowdsourcing \cite{yan2014learning,chen2017learning} and online queries \cite{schroff2011harvesting,divvala2014learning} inexpensively obtain data, but unavoidably yield noisy labels. Training with too many noisy labels reduces generalization performance of DNNs since the networks can easily overfit on corrupted labels \cite{zhang2016understanding,arpit2017closer}. To utilize extensive noisy data, understanding how noisy labels affect training and generalization of DNNs is the very first step, based on which we can design specific methods to train DNNs robustly in practical applications.

Numerous methods have been proposed to deal with noisy labels. Several methods focus on estimating the noise transition matrix and correcting the objective function accordingly, e.g., forward or backward correction \cite{patrini2017making}, S-model \cite{goldberger2017training}. However, it is a challenge to estimate the noise transition matrix accurately. An alternative approach is training on selected or weighted samples, e.g., Decoupling \cite{malach2017decoupling}, MentorNet \cite{jiang2018mentornet}, gradient-based reweighting \cite{ren2018learning} and Co-teaching \cite{han2018co}. A remaining issue is to design a reliable and convincing criteria of selecting or weighting samples. Another approach proposes to correct labels using the predictions of DNNs, e.g., Bootstrap \cite{reed2015training}, Joint Optimization \cite{tanaka2018joint} and D2L \cite{ma2018dimensionality}, all of which are vulnerable to overfitting. To improve the robustness, Joint Optimization introduces regularization terms requiring a prior knowledge of how actual classes distribute among all training samples. However, the prior knowledge is usually unavailable in practice.

How noisy labels affect training and generalization of DNNs is not well understood, which deserves more attention since it may promote fundamental approaches of robustly training DNNs against noise. Without label corruption, the generalization error can be bounded by complexity measures such as VC dimension \cite{vapnik1998adaptive}, Rademacher complexity \cite{bartlett2002rademacher} and uniform stability \cite{mukherjee2002statistical,bousquet2002stability,poggio2004general}. But the bounds become trivial in the presence of noisy labels. \citet{zhang2016understanding} demonstrated that DNNs have the high capacity to fit even random labels, but obtain a large generalization error. \citet{zhang2016understanding} also showed a positive correlation between generalization error and noise ratio, which implies DNNs do capture some useful information out of the noisy data. \citet{arpit2017closer} showed that during training, DNNs tend to learn simple patterns first, then gradually memorize all samples, which justifies the widely used \textit{small-loss criteria}: treating samples with small training loss as clean ones \cite{han2018co,jiang2018mentornet}. \citet{ma2018dimensionality} qualitatively attributed the poor generalization performance of DNNs to the increased dimensionality of the latent feature subspace. Through extensive experiments, these works gained empirical insight into the interesting behavior of DNNs trained with noisy labels, while a theoretical and quantitative explanation is yet to emerge. 

In this paper, we can quantitatively clarify the generalization performance of DNNs normally trained with noisy labels. To verify our theoretical analysis, we apply cross-validation to randomly split a set of collected samples, whose labels may be polluted by some noise. DNNs can be trained on a subset, then evaluated on the remaining dataset to compare the theoretically and empirical results on the generalization performance. We find that DNNs can fit noisy training sets exactly and generalize in distribution (see Claim~\ref{claim1} for more details). Hence, we can quantitatively characterize the test accuracy in terms of noise ratio in datasets. In particular, the test accuracy is a quadratic function of the noise ratio in the case of symmetric noise. In \citet{zhang2016understanding}, it has been empirically found that the generalization performance of DNNs is highly dependent on the noise ratio. One of our contributions is to provide a thorough explanation for their empirical findings.

Based on our analysis, we further develop a specific method to train DNNs against noisy labels. Our method is developed on top of the Co-teaching strategy, which is first presented in \citet{blum1998combining} and then modified to deal with noisy labels with impressive performance in \cite{han2018co}. In the Co-teaching strategy, one trains two networks simultaneously: mini-batches are drawn from the whole noisy training set, then each network selects a certain number of small-loss samples and feeds them to its peer network. However, the performance of the Co-teaching decays seriously when the noise ratio of the training set increases. Moreover, the number of small-loss samples selected in each mini-batch is set according to the noise ratio of the training set, which is unavailable in practice. Fortunately, we can address these issues based on our theoretical analysis on the generalization performance of DNNs. Specially, we present the Iterative Noisy Cross-Validation (INCV) method to select a subset of samples, which has much smaller noise ratio than the original dataset, resulting in a more stable training process of DNNs. Moreover, we can automatically estimate the noise ratio of the selected set, which makes our method more practical for industrial applications. Briefly speaking, our main contributions are
\begin{itemize}%[noitemsep]
	\item theoretically relating the generalization performance of DNNs to the label noise,
	\item practical algorithms of selecting clean labels and training noise-robust DNNs.	
\end{itemize}

Experiments on both synthetic and real-world noisy labels show that compared with state-of-the-art methods \cite{patrini2017making,malach2017decoupling,han2018co,jiang2018mentornet,ma2018dimensionality}, DNNs trained using our strategy achieve the best test accuracy on the clean test set. In particular, our method is verified on (i) the CIFAR-10 dataset \cite{krizhevsky2009learning} with synthetic noisy labels generated by randomly flipping the original ones, and (ii) the WebVision dataset \cite{li2017webvision}, which is a large benchmark consisting of 2.4 million images crawled from websites, containing real-world noisy labels.

\section{Preliminaries}
\label{Sec_Preliminaries}
For a $c$-class classification, we collect a dataset $\mathcal{D}=\{x_t,y_t\}_{t=1}^n$, where $x_t$ is the $t$-th sample with its observed label as $y_t\in[c]:=\{1,\ldots,c\}$. As discussed previously, the observed label $y$ may be corrupted since the example $x$ are often labeled by online queries or in crowdsourcing system. Let $\hat{y}$ denote the true label, we can describe the corruption process of the set $\mathcal{D}$ by introducing a noise transition matrix $T\in\mathbb{R}^{c\times c}$, where $T_{ij}=P(y=j|\hat{y}=i)$ denotes the probability of labeling an $i$-th class example as $j$. In the cross-validation, we randomly split the collected samples $\mathcal{D}$ into two halves $\mathcal{D}_1$ and $\mathcal{D}_2$. In this way, $\mathcal{D}_2$ shares the same noise transition matrix $T$ with $\mathcal{D}_1$. Let $f(x;\omega)$ denote a neural network parameterized by $\omega$, and $y^f\in[c]$ denote the predicted label of $x$ given by the network $f(x;\omega)$. %We aim to optimize the parameter $\omega$ on the training set $\mathcal{D}_1$, so that the learned network would generalize well on the test set $\mathcal{D}_2$.

\section{Understanding DNNs trained with noisy labels}
\label{Sec_method}
Extensive experiments in \cite{zhang2016understanding} have shown that DNNs can fit the noisy, even random, labels contained in the training set, but the generalization error is large even on a test set with the same noise. In this section, we use the previously introduced noise transition matrix $T$ to theoretically quantify the generalization performance of DNNs normally trained with noisy labels, which perfectly explains the empirical findings reported in \cite{zhang2016understanding}.

%DNNs normally trained on a noisy dataset will eventually overfit on those corrupted labels, resulting in large generalize error.
%%To deal with label corruption, figuring out what does a network learn from a corrupted dataset is the very first step.
%In this section, we use the noise transition matrix $T$ to quantify the generalization performance of DNNs trained with noisy labels.

%\subsection{Generalization of DNNs trained with noisy labels}

In the classical \textit{Probably Approximately Correct} framework \cite{valiant1984theory}, good generalization performance means that prediction $y^f$ and observed test label $y$ are approximately identical as random variables, namely they should be equal for each testing sample $x$. Without label corruption, the generalization error can be bounded by VC dimension \cite{vapnik1998adaptive}, Rademacher complexity \cite{bartlett2002rademacher}, etc. However, in dealing with DNNs trained with noisy labels, $y^f=y$ possibly does not hold when evaluated at each testing example $x$, resulting in a large generalization error \cite{zhang2016understanding}. Fortunately, we find that the generalization still occurs in the sense of distribution, namely \emph{generalization in distribution}, as shown in the following Claim \ref{claim1}. Recall that in cross-validation, we randomly divide a noisy dataset $\mathcal{D}$ into two halves $\mathcal{D}_1$ and $\mathcal{D}_2$.
\begin{claim}
	\label{claim1} (Generalization in distribution). Let $f(x;\omega)$ be the network trained on $\mathcal{D}_1$ and tested on $\mathcal{D}_2$. If we assume \\
	(i) the observed input examples $x$ are i.i.d. in the set $\mathcal{D}$,\\
	(ii) $f$ has a sufficiently high capacity,\\%the high capacity to fit the training set, $\mathcal{D}_1$, \\
	then on $\mathcal{D}_2$, the probability of predicting an truly $i$-th class test sample as $j$ is
	\begin{equation}
	\label{Eq_ge}
	P(y^{f}=j|\hat{y}=i)=T_{ij},
	\end{equation}	
	where $T_{ij}:=P(y=j|\hat{y}=i)$ denotes the noise transition matrix shared by $\mathcal{D}_1$ and $\mathcal{D}_2$.
\end{claim}
Claim \ref{claim1} reveals the fact that the prediction $y^{f}$ and the test label $y$ have the same distribution. Actually, if the model trained on $\mathcal{D}_1$ is tested on another clean test set with true labels, Eq.~(\ref{Eq_ge}) still holds, while in this case it implies that the probability of predicting an $i$-th class test sample as $j$ equals to the $T_{ij}$ of the training set $\mathcal{D}_1$. We will justify the Claim~\ref{claim1} through experiments in Sec.~\ref{Sec_exp_verify}.

The \textbf{Test Accuracy} is a widely used metric, which is defined as the proportion of testing examples for which the \textit{prediction} $y^f$ equals to the\textit{ observed label} $y$. In the following Prop.~\ref{Prop_ge}, we formulate the test accuracy on the test set $\mathcal{D}_2$.
\begin{prop}
	\label{Prop_ge} Let $\mathcal{D}_1$ and $\mathcal{D}_2$ be two datasets with the same noise transition matrix $T$, $f(x;\omega)$ be a network trained on $\mathcal{D}_1$ and tested on $\mathcal{D}_2$. Following the assumptions in Claim~\ref{claim1}, the test accuracy for any class $i\in[c]$ is
	\begin{equation}
	\label{Eq_ge_cor}
	P(y^{f}=y|\hat{y}=i)=\sum_{j=1}^{c}T_{ij}^{2}.
	\end{equation}
\end{prop}
\begin{proof}
	Based on Claim \ref{claim1}, $y^{f}$ and $y$ have the same distribution characterized by $T$. Assume the label corruption process is independent, then on the test set, we have
	\begin{equation}
	\label{Eq_ge_cor_detailed}
	\begin{aligned}
	&P(y^{f}=j,y=k|\hat{y}=i)\\
	=&P(y^{f}=j|\hat{y}=i)P(y=k|\hat{y}=i) = T_{ij}T_{ik}.
	\end{aligned}
	\end{equation}
	Hence, Eq.~(\ref{Eq_ge_cor}) follows from $P(y^{f}=y|\hat{y}=i)=\sum_{j=1}^{c}P(y^{f}=j,y=j|\hat{y}=i)$.
\end{proof}

%Therefore, it is not difficult to see that the probability of $y^{f}=y$ when evaluated at each example is a summation of $T_{ij}^{2}$, $j=1,\cdots,c$. For the reader's convenience, we give a simple proof here.
%\begin{proof}
%	$\forall i,j,k\in\{1,\ldots,c\}$, 
%	\begin{equation}
%	%\label{Eq_ge_cor_detailed}
%	\begin{aligned}
%	&P(y^{f}=j,y=k|\hat{y}=i)\\
%	=&P(y^{f}=j|\hat{y}=i)P(y=k|\hat{y}=i) = T_{ij}T_{ik}.
%	\end{aligned}
%	\end{equation}
%	Hence, the desired result follows from $P(y^{f}=y|\hat{y}=i)=\sum_{j=1}^{c}P(y^{f}=j,y=j|\hat{y}=i)$.
%%	\begin{equation}
%%	\nonumber
%%	P(y^{f}=y|\hat{y}=i)=\sum_{j=1}^{c}P(y^{f}=j,y=j|\hat{y}=i).
%%	\end{equation}
%\end{proof}

\subsection{Symmetric and Asymmetric Noise}
Following previous literatures \cite{ren2018learning,han2018co,jiang2018mentornet,ma2018dimensionality}, in this subsection we focus on investigating two representative types of noise, symmetric and asymmetric noise, which can be defined as follows (see Fig.~\ref{Fig_T} for examples), 
\begin{defi}
	\label{def_sym_asym}
	In the case of \textbf{symmetric noise} of ratio $\varepsilon$, $\forall i\in[c]$, we define $T_{ii}=1-\varepsilon$, and $T_{ij}=\varepsilon/(c-1), \forall j\neq i$. 
	\\In the case of \textbf{asymmetric noise} of ratio $\varepsilon$,  $\forall i\in[c]$, we define $T_{ii}=1-\varepsilon$, $T_{ij}=\varepsilon$ for some $j\neq i$, and $T_{ij}=0$ otherwise.
\end{defi}

In the cases of symmetric and asymmetric noise, we can use the noise ratio $\varepsilon$ to quantify the test accuracy of DNNs, which are trained and tested on previously mentioned noisy datasets $\mathcal{D}_1$ and $\mathcal{D}_2$, respectively.
\begin{cor}
	\label{corollary1}
	For symmetric noise of ratio $\varepsilon$, the test accuracy is
	\begin{equation}
	\label{Eq_sym_noise}
	P(y^{f}=y)=(1-\varepsilon)^2+\frac{\varepsilon^2}{c-1}.
	\end{equation}
	For asymmetric noise of ratio $\varepsilon$,  the test accuracy is
	\begin{equation}
	\label{Eq_asym_noise}
	P(y^{f}=y)=(1-\varepsilon)^2+\varepsilon^2.
	\end{equation}
\end{cor}

\begin{proof}
	Following Prop.~\ref{Prop_ge}, we have
	\begin{equation}
	\begin{aligned}
	\nonumber
	P(y^f=y)&=\sum_{i=1}^{c}P(\hat{y}=i)P(y^f=y|\hat{y}=i)\\
	&=\sum_{i=1}^{c}P(\hat{y}=i)\sum_{j=1}^{c}T_{ij}^{2}.		
	\end{aligned}
	\end{equation}
	Note that for the symmetric and asymmetric noise, $\forall i\in[c]$, $\sum_{j=1}^{c}T_{ij}^{2}$ is a constant given by $\varepsilon$. Therefore, the desired result follows by inserting $\varepsilon$ into the equation.
\end{proof}
Interestingly, Eq.~(\ref{Eq_sym_noise}) perfectly fits the experimental results of generalization accuracy shown in Fig.~1(c) of \cite{zhang2016understanding}, and 
%provides a chance to automatically 
enables us to
estimate the noise ratio of a dataset from the experimental test accuracy.

\begin{algorithm}[t]
	\caption{Noisy Cross-Validation (NCV): selecting clean samples out of the noisy ones}
	\label{Alg1}
	\textbf{INPUT:} the noisy set $\mathcal{D}$, epoch $E$
	\begin{algorithmic}[1]
		\STATE $\mathcal{S}=\emptyset$, initialize a network $f(x;\omega)$
		\STATE Randomly divide $\mathcal{D}$ into two halves $\mathcal{D}_1$ and $\mathcal{D}_2$
		\STATE Train $f(x;\omega)$ on $\mathcal{D}_1$ for $E$ epochs
		\STATE Select samples, $\mathcal{S}_1=\{(x,y)\in\mathcal{D}_2:y^{f}=y\}$
		\STATE Reinitialize the network $f(x;\omega)$
		\STATE Train $f(x;\omega)$ on $\mathcal{D}_2$  for $E$ epochs
		\STATE Select samples, $\mathcal{S}_2=\{(x,y)\in\mathcal{D}_1:y^{f}=y\}$
		\STATE $\mathcal{S}=\mathcal{S}_1\cup\mathcal{S}_2$%, $\hat{\mathcal{D}}=\hat{\mathcal{D}}-\mathcal{S}_1\cup\mathcal{S}_2$
	\end{algorithmic}
	\textbf{OUTPUT:} the selected set $\mathcal{S}$%, remaining set $\hat{\mathcal{D}}$
\end{algorithm}

\section{Training DNNs against noisy labels}
\label{Sec_identify}
In this section, we present a method on top of the Co-teaching strategy to train DNNs robustly against noisy labels. As introduced previously, the performance of the Co-teaching decays seriously and becomes unstable when the noise ratio of the training set increases, which is further demonstrated in our experiments. To address this issue, we propose to first select a subset of samples, which has much smaller noise ratio than the original dataset.

A sample $(x,y)$ is \textbf{clean}, if its observed label $y$ equals to its latent true class $\hat{y}$. However, $\hat{y}$ is unavailable in practice. We propose to identify a sample $(x,y)$ as {clean} if its observed label $y$ equals to its predicted label $y^f$ given by the network $f(x;\omega)$. If we aim to identify whether a sample $(x,y)$ is clean or not, we should keep this sample out of the training set.  
An intuitive method can be found in Alg. \ref{Alg1}, namely the Noisy Cross-Validation (NCV) method, whose validity will be justified through the following theoretical analysis and extensive experiments in the next section. 

Following the standard metrics \cite{powers2011evaluation}, we measure the 
identification performance in terms of \textit{Label Precision} ($LP$) \cite{han2018co} and \textit{Label Recall} ($LR$),
\begin{equation}
\label{Eq_def_LPLR}
\begin{aligned}
LP:=\frac{\lvert\{(x,y)\in\mathcal{S}:y=\hat{y}\}\rvert}{\lvert\mathcal{S}\rvert},\\
LR:=\frac{\lvert\{(x,y)\in\mathcal{S}:y=\hat{y}\}\rvert}{\lvert\{(x,y)\in\mathcal{D}:y=\hat{y}\}\rvert},
\end{aligned}
\end{equation}
where $\mathcal{S}\subset\mathcal{D}$ is the selected subset as given in Alg~\ref{Alg1}, and $\lvert\cdot\rvert$ denotes the number of samples in a set. In this way, $LP$ represents the fraction of clean samples in $\mathcal{S}$, and $LR$ represents the fraction of clean samples in $\mathcal{S}$ over all clean samples in $\mathcal{D}$. Note that the noise ratio of the selected set $\mathcal{S}$ is $\varepsilon_{\mathcal{S}}=1-LP$ according to the above definition.
We also have $LP$ and $LR$ for any class $i\in[c]$:
\begin{equation}
\begin{aligned}
LP_i:=\frac{\lvert\{(x,y)\in\mathcal{S}:y=\hat{y}=i\}\rvert}{\lvert\{(x,y)\in\mathcal{S}:\hat{y}=i\}\rvert},\\
LR_i:=\frac{\lvert\{(x,y)\in\mathcal{S}:y=\hat{y}=i\}\rvert}{\lvert\{(x,y)\in\mathcal{D}:y=\hat{y}=i\}\rvert}.
\end{aligned}
\label{Eq_LPLRi}
\end{equation}

Based on the analysis presented in Sec.~\ref{Sec_method}, we quantify the performance of Alg.~\ref{Alg1} in the following Prop.~\ref{Prop_LPLR}.
%Since $LP$ is a weighted average of $LP_i$ and similarly for $LR$ and $LR_i$, it suffices to examine $LP_i$ and $LR_i$ specifically for any class $i$.
\begin{prop}
	\label{Prop_LPLR}
	Using Alg.~\ref{Alg1} to select clean samples, we have, $\forall i\in[c]$
	\begin{equation}
	\label{Eq_LPLR}
	\begin{aligned}
	LP_i=\frac{T_{ii}^{2}}{\sum_{j=1}^{c}T_{ij}^{2}},\quad LR_i=T_{ii}.
	\end{aligned}
	\end{equation}
\end{prop}
\begin{proof}
	According to Alg. \ref{Alg1}, we can reformulate Eq. (\ref{Eq_LPLRi}) as
	\begin{equation}
	\begin{aligned}
	\nonumber
	&LP_i = \frac{P(y^{f}=i,y=i|\hat{y}=i)}{P(y^{f}=y|\hat{y}=i)},\\
	&LR_i = \frac{P(y^{f}=i,y=i|\hat{y}=i)}{P(y=i|\hat{y}=i)}.
	\end{aligned}
	\end{equation}
	The desired result follows by inserting Eq.~(\ref{Eq_ge_cor}) $\&$ (\ref{Eq_ge_cor_detailed}) into the above equations.
\end{proof}

\subsection{Symmetric and Asymmetric Noise}
Since $\forall i$, $\sum_{j=1}^{c}T_{ij}=1$, Eq.~(\ref{Eq_LPLR}) in general implies:
\begin{cor}
	\begin{equation}
	\label{Eq_bound}
	\begin{aligned}
	\frac{T_{ii}^{2}}{T_{ii}^{2}+(1-T_{ii})^{2}}\leq LP_i\leq\frac{T_{ii}^{2}}{T_{ii}^{2}+\frac{(1-T_{ii})^{2}}{c-1}}.
	\end{aligned}
	\end{equation}
\end{cor}
Interestingly, we can see that the upper bound of Eq.~(\ref{Eq_bound}) is attained for the symmetric noise, and the lower bound is attained for the asymmetric noise. In the cases of symmetric and asymmetric noise, we further have  $LP=LP_1=\cdots=LP_c$, $LR=LR_1=\cdots=LR_c$, so that we can reformulate the $LP$ and $LR$ in the following Cor.~\ref{Cor_LPLR}.

\begin{algorithm}[t]
	\caption{Iterative Noisy Cross-Validation (INCV): selecting clean samples out of the noisy ones} 
	\label{Alg2}
	\textbf{INPUT:} the noisy set $\mathcal{D}$, number of iterations $N$, epoch $E$, remove ratio $r$
	\begin{algorithmic}[1]
		\STATE selected set $\mathcal{S}=\emptyset$, candidate set $\mathcal{C}=\mathcal{D}$
		\FOR {$i=1,\cdots,N$}
		\STATE Initialize a network $f(x;\omega)$
		\STATE Randomly divide $\mathcal{C}$ into two halves $\mathcal{C}_1$ and $\mathcal{C}_2$
		\STATE Train $f(x;\omega)$ on $\mathcal{S}\cup\mathcal{C}_1$ for $E$ epochs
		\STATE Select samples, $\mathcal{S}_1=\{(x,y)\in\mathcal{C}_2:y^{f}=y\}$
		%		\IF{$i=1$}
		%		\STATE Estimate the noise ratio $\varepsilon$ using Eq.~(\ref{Eq_sym_noise})
		%		\ENDIF
		\STATE Identify $n=r\lvert\mathcal{S}_1\rvert$ samples that will be removed:\\
		\centerline{$\mathcal{R}_1=\{\#n\arg\max_{\mathcal{C}_2}\mathcal{L}(y,f(x;\omega))\}$}
		\STATE \textbf{if} $i=1$, estimate the noise ratio $\varepsilon$ using Eq.~(\ref{Eq_sym_noise})
		\STATE Reinitialize the network $f(x;\omega)$
		\STATE Train $f(x;\omega)$ on $\mathcal{S}\cup\mathcal{C}_2$ for $E$ epochs
		\STATE Select samples, $\mathcal{S}_2=\{(x,y)\in\mathcal{C}_1:y^{f}=y\}$
		\STATE Identify $n=r\lvert\mathcal{S}_2\rvert$ samples that will be removed:\\
		\centerline{$\mathcal{R}_2=\{\#n\arg\max_{\mathcal{C}_1}\mathcal{L}(y,f(x;\omega))\}$}
		\STATE $\mathcal{S}=\mathcal{S}\cup\mathcal{S}_1\cup\mathcal{S}_2$, $\mathcal{C}=\mathcal{C}-\mathcal{S}_1\cup\mathcal{S}_2\cup\mathcal{R}_1\cup\mathcal{R}_2$
		\ENDFOR
	\end{algorithmic}
	\textbf{OUTPUT:} the selected set $\mathcal{S}$, remaining candidate set $\mathcal{C}$ and estimated noise ratio $\varepsilon$	
\end{algorithm}

\begin{cor}
	\label{Cor_LPLR}
	For the symmetric noise of ratio $\varepsilon$, we have
	\begin{equation}
	\label{EQ_LP}
	\begin{aligned}
	LP=\frac{(1-\varepsilon)^{2}}{(1-\varepsilon)^{2}+\varepsilon^2/(c-1)},\quad LR=1-\varepsilon.
	\end{aligned}
	\end{equation}
	For the asymmetric noise of  ratio $\varepsilon$, we have
	\begin{equation}
	\label{EQ_LR}
	\begin{aligned}
	LP=\frac{(1-\varepsilon)^{2}}{(1-\varepsilon)^{2}+\varepsilon^2},\quad LR=1-\varepsilon.
	\end{aligned}
	\end{equation}	
\end{cor}
Given the noise ratio $\varepsilon$ of the original set $\mathcal{D}$ estimated by Eq.~(\ref{Eq_sym_noise}) or (\ref{Eq_asym_noise}), the above Cor.~\ref{Cor_LPLR} further enables us to estimate the metrics $LP$ and $LR$. Recall that the noise ratio of the selected subset $\mathcal{S}$ is $\varepsilon_S=1-LP$ according to the definition of $LP$. In practical situations ($\forall i$, $T_{ii}$ being the largest among $T_{ij}$, $j\in[c]$), \textbf{Alg.~\ref{Alg1} always produces a subset with smaller noise ratio} $\varepsilon_S < \varepsilon$. See Supp. D for more details.
\subsection{Improving the Co-teaching with the INCV method}
Although the subset selected by Alg.~\ref{Alg1} usually has much smaller noise ratio than the original set, the robust training of DNNs may require larger number of training samples. To address this issue, we present the Iterative Noisy Cross-Validation (INCV) method to increase the number of selected samples by applying Alg.~\ref{Alg1} iteratively. More details of the INCV can be found in Alg.~\ref{Alg2}. Apart from selecting clean samples, the INCV removes samples that have large categorical cross entropy loss at each iteration. The remove ratio $r$ determines how many samples will be removed.

After a detailed dissection of the noisy dataset $\mathcal{D}$ by Alg.~\ref{Alg2}, we can further improve the Co-teaching to take full advantage of the selected set $\mathcal{S}$ and the candidate set $\mathcal{C}$. Specifically, we let the two networks focus on the selected set $\mathcal{S}$ at the first $E_0$ epochs, then incorporate the candidate set $\mathcal{C}$. Hence, both training stability and test accuracy are improved. More details of our method can be found in Alg.~\ref{Alg3}.

\begin{algorithm}[t]
	\caption{Training DNNs robustly against noisy labels} 
	\label{Alg3}
	\textbf{INPUT:} the selected set $\mathcal{S}$, candidate set $\mathcal{C}$ and estimated noise ratio $\varepsilon$ from Alg.~\ref{Alg2}, warm-up epoch $E_0$, total epoch $E_{max}$
	\begin{algorithmic}[1]
		\STATE Initialize two networks $f_1(x;\omega_1)$ and $f_2(x;\omega_2)$
		\FOR {$e=1,\cdots,E_{max}$}
		\FOR {batches $(\mathcal{B}_{\mathcal{S}},\,\mathcal{B}_{\mathcal{C}})$ in $(\mathcal{S},\,\mathcal{C})$}%{$i=1,\cdots,nStep$}
		%\STATE Draw mini-batch $\bar{\mathcal{S}}$, $\bar{\mathcal{C}}$ from $\mathcal{S}$, $\mathcal{C}$
		\STATE \textbf{if} $t > E_0$ \textbf{then} $\mathcal{B}=\mathcal{B}_{\mathcal{S}}\cup\mathcal{B}_{\mathcal{C}}$, \textbf{else} $\mathcal{B}=\mathcal{B}_{\mathcal{S}}$
		%		\IF{$t > E_0$}
		%		\STATE $\mathcal{B}=\mathcal{B}_{\mathcal{S}}\cup\mathcal{B}_{\mathcal{C}}$%$\mathcal{B}=\bar{\mathcal{S}}$
		%		\ELSE
		%		\STATE $\mathcal{B}=\mathcal{B}_{\mathcal{S}}$%$\mathcal{B}=\bar{\mathcal{S}}\cup\bar{\mathcal{C}}$
		%		\ENDIF
		\STATE $\mathcal{B}_1=\{\#n(e)\arg\min_{\mathcal{B}}\mathcal{L}(y,f_1(x;\omega_1))\}$
		\STATE $\mathcal{B}_2=\{\#n(e)\arg\min_{\mathcal{B}}\mathcal{L}(y,f_2(x;\omega_2))\}$
		\STATE Update $f_1$ using $\mathcal{B}_2$
		\STATE Update $f_2$ using $\mathcal{B}_1$
		\ENDFOR
		\ENDFOR
	\end{algorithmic}
	\textbf{OUTPUT:} $f_1(x;\omega_1)$, $f_2(x;\omega_2)$
\end{algorithm}

\section{Experiments}
\label{Sec_exp}
This section consists of three parts. Firstly, we experimentally verify the theoretical results presented in Sec.~\ref{Sec_method} $\&$ \ref{Sec_identify}. Then we demonstrate that the INCV method shown in Alg.~\ref{Alg2} can identify more samples that have correct labels. Finally, we show that our proposed method outlined in Alg.~\ref{Alg3} can train DNNs robustly against noisy labels, and outperforms state-of-the-art methods \cite{patrini2017making,malach2017decoupling,han2018co,jiang2018mentornet,ma2018dimensionality}. Our code is available at \url{https://github.com/chenpf1025/noisy_label_understanding_utilizing}.

\textbf{Experimental setup.} To verify our theory and test the algorithm, we first conduct experiments on synthetic noisy labels generated by randomly corrupting the original labels in CIFAR-10 \cite{krizhevsky2009learning}. We focus on two representative types of noise: symmetric noise and asymmetric noise, as defined in Def.~\ref{def_sym_asym} and illustrated in Fig~\ref{Fig_T}. To verify our method on real-world noisy labels, we use the WebVision dataset \cite{li2017webvision} which contains 2.4 million images crawled from websites using the 1,000 concepts in ImageNet ILSVRC12 \cite{deng2009imagenet}. The training set of WebVision contains many real-world noisy labels without human annotation. More implementation details are presented in Supp.~A. In the following subsections, we focus on experimental results and discussions.

\begin{figure}[t]	
	%\vskip -0.2in
	\begin{center}
		\centerline{\includegraphics[width=0.8\columnwidth]{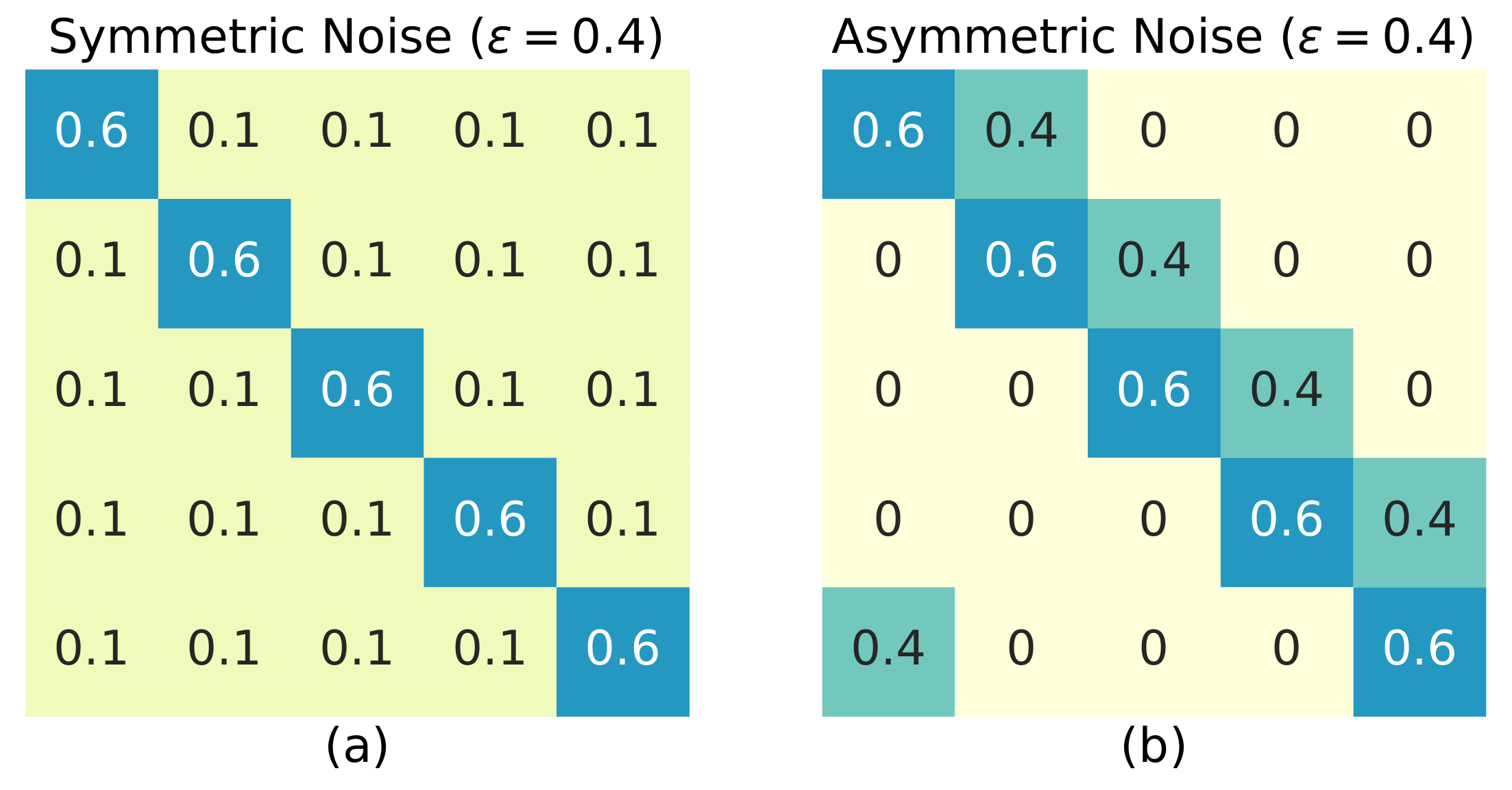}}
		\vskip -0.15in
		\caption{Examples of noise transition matrix $T$ (taking 5 classes and noise ratio $0.4$ as an example).}
		\label{Fig_T}
	\end{center}
	\vskip -0.3in
\end{figure}

\subsection{Behavior of DNNs trained with noisy labels}
\begin{figure*}[htp]	
	\vskip -0.1in
	\begin{center}
		\centerline{\includegraphics[width=1.6\columnwidth]{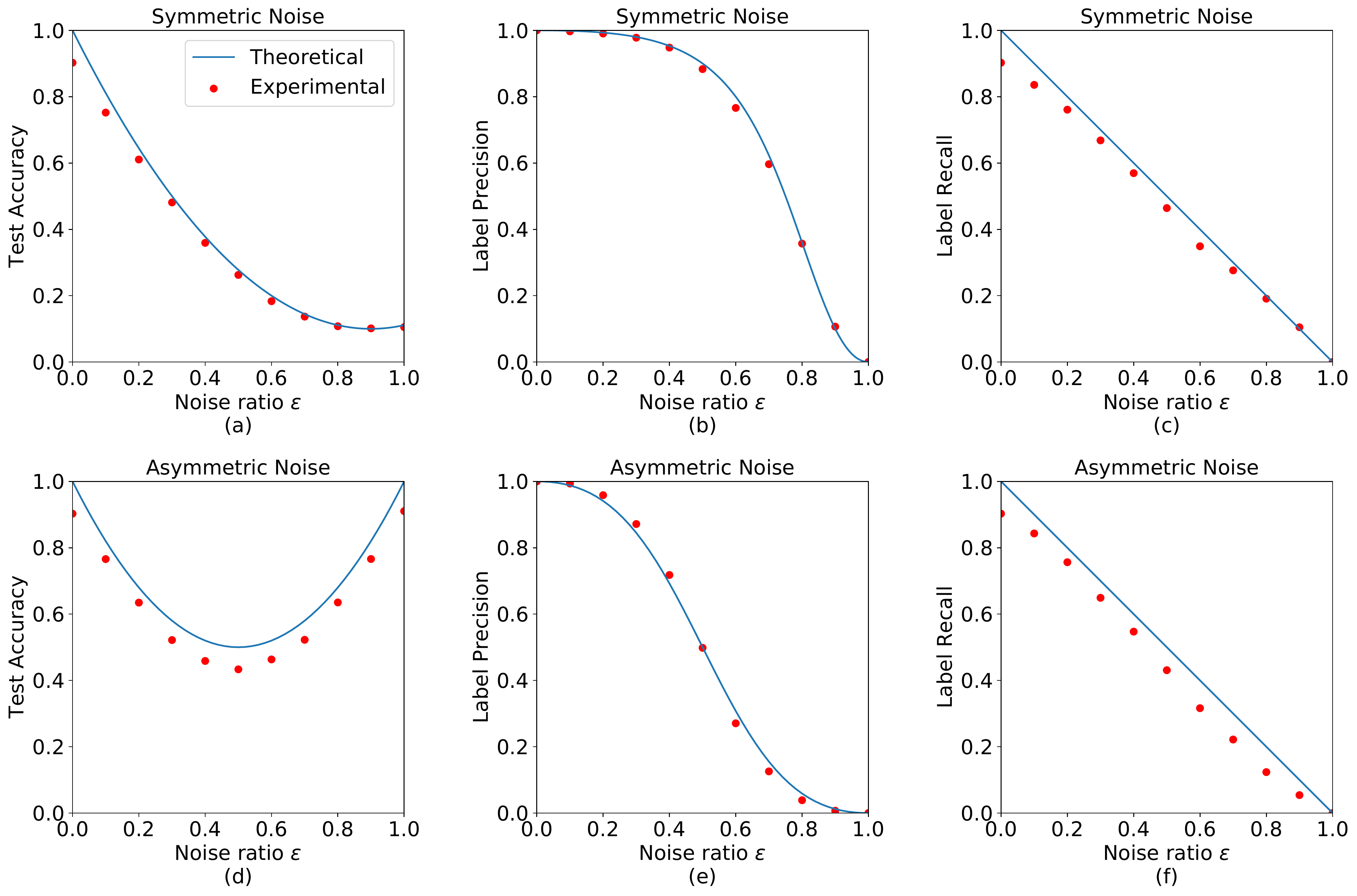}}
		\vskip -0.15in
		\caption{Test accuracy, label precision ($LP$) and label recall ($LR$) w.r.t noise ratio on manually corrupted CIFAR-10. The first row corresponds to symmetric noise and the second row asymmetric. Following cross-validation, we train the ResNet-110 on half of the noisy dataset and test on the rest half. The experimental results are consistent with the theoretical curves.}
		\label{Fig_LPLR}
	\end{center}
	\vskip -0.2in
\end{figure*}
\begin{figure}[t]
	\vskip -0.1in
	\begin{center}
		\centerline{\includegraphics[width=\columnwidth]{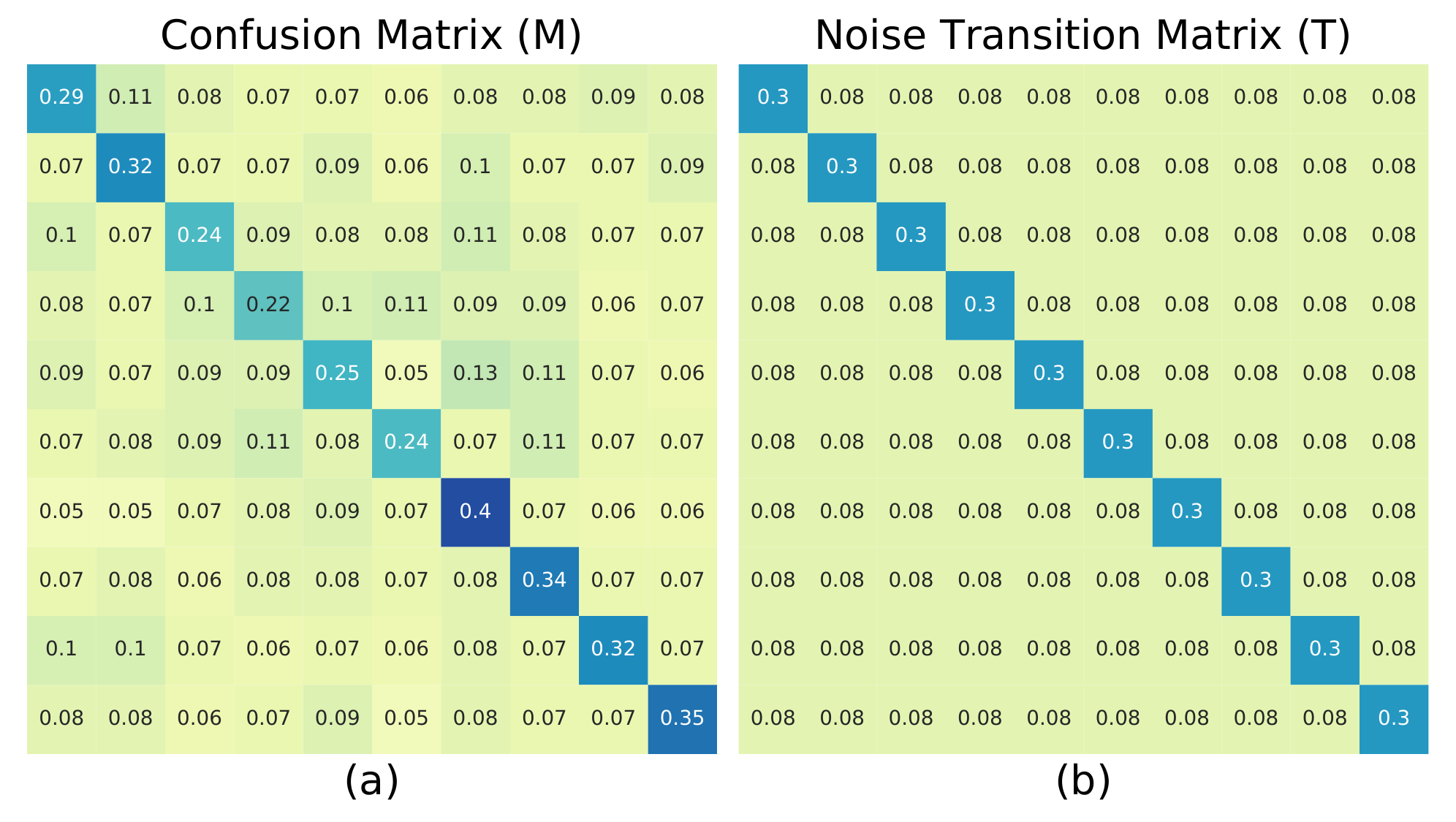}}
		\vskip -0.15in
		\caption{Confusion matrix of the RseNet-110 which is normally trained on manually corrupted CIFAR-10 with noise transition matrix $T$. $M\approx T$ satisfies the statement presented in Claim~\ref{claim1}.
			%Due to lack of space, here we simply show results for symmetric noise of ratio $0.7$, and more plots are shown Supp.~B.
		}
		\label{Fig_Confusion}
	\end{center}
	\vskip -0.4in
\end{figure}

\label{Sec_exp_verify}
For DNNs normally trained with noisy labels, we have theoretically characterized their behavior with the following metrics (i) test accuracy given in Eq.~(\ref{Eq_sym_noise}) $\&$ (\ref{Eq_asym_noise}), (ii) $LP$ given in Eq.~(\ref{EQ_LP}) $\&$ (\ref{EQ_LR}); (iii) $LR$ given in Eq.~(\ref{EQ_LP}) $\&$ (\ref{EQ_LR}). In this subsection, we evaluate these three metrics in extensive experiments, and show that experimental results confirm our theoretical analysis. Given a noisy dataset $\mathcal{D}$, we implement cross-validation to randomly split it into two halves $\mathcal{D}_1$, $\mathcal{D}_2$, then train the ResNet-110 \cite{he2016identity} on $\mathcal{D}_1$ and test on $\mathcal{D}_2$.

\textbf{Experimental results confirm the theoretical analysis.} As shown in Fig.~\ref{Fig_LPLR}, the experimental results are consistent with theoretical estimations. In particular, Fig.~\ref{Fig_LPLR} (a) reproduces the observation shown in \cite{zhang2016understanding} that the test accuracy is highly dependent of the noise ratio. \cite{zhang2016understanding} did not present any theoretical explanations while we explicitly formulate in Eq.~(\ref{Eq_sym_noise}) that the test accuracy is a quadratic function of the noise ratio. In Fig~\ref{Fig_LPLR} (b) and (e), the experimental $LP$ is precisely given by our formulas. %, which implies the theoretical guarantee we provide for Alg.~\ref{Alg1} is effective: if we run Alg.~\ref{Alg1} to select clean samples, the noise ratio of the selected subset is $\varepsilon_{\mathcal{S}}=1-LP$.
It is observed that for some data points, the experimental test accuracy and $LR$ are slightly smaller than our theoretical values. This is reasonable since the distribution of $\mathcal{D}_2$ is not exactly the same as $\mathcal{D}_1$, and the generalization error would not become $0$ even without noise.

To further investigate the prediction behavior of DNNs trained with noisy labels, we define a confusion matrix $M$, whose $ij$-th entry represents the probability of predicting an $i$-th class test sample as $j$, s.t.,
\begin{equation}
\nonumber
M_{ij}:=P(y^{f}=j|\hat{y}=i).
\end{equation}
Fig.~\ref{Fig_Confusion} illustrates the confusion matrix of DNNs trained on manually corrupted CIFAR-10 with symmetric noise of ratio $0.7$, and we can find that $M\approx T$, which satisfies the statement presented in Claim~\ref{claim1}. More results can be found in Supp.~B, where we show $M\approx T$ still holds.

\textbf{Training accuracy converging to an extremely low value does not contradict our findings}. We find that under large symmetric noise, training accuracy of the model always converges to an extremely low value. In the experiments, when trained with symmetric noise of ratio $0.7$, $0.8$, $0.9$ and $1.0$, the training accuracies are only $0.58$, $0.40$, $0.24$ and $0.36$, respectively. However, we show in Fig.~\ref{Fig_LPLR} $\&$ \ref{Fig_Confusion} that our theoretical results are always consistent with the experimental ones. The phenomena further raises a fundamental question: \textit{Is a high training accuracy a necessary condition of learning and generalization?}
%Fig.~\ref{Fig_LPLR} and Fig~\ref{Fig_Confusion} indicate the answer is \textit{No}.
Without data augmentation, the theorem on finite sample expressiveness \cite{zhang2016understanding} indicates that DNNs can always achieve $0$ training error on the finite number of training samples. However, standard data augmentation \cite{he2016deep} is used in our implementation, which makes it difficult to achieve a high training accuracy, especially under large symmetric noise. Intuitively, due to the existence of noisy labels, nearby samples from the same class may have different labels, requiring many small regions to be classified differently. Augmentation easily generates random samples violating the classifier regions learned previously, hence increases the training error. Even in this case, our theoretical formulas presented previously still hold, as shown in Fig.~\ref{Fig_LPLR} $\&$ \ref{Fig_Confusion}. Here we conclude that \textit{as long as a sufficiently rich deep neural network is trained for sufficiently many steps till convergence, the network can fit the training set and generalize in distribution, even if there are noisy labels and the training accuracy is low.} We call for more theoretical explanations on this interesting phenomena in future.

\subsection{Identifying more clean samples by the INCV}
\label{Sec_exp_id}
Fig.~\ref{Fig_LPLR} (b) and (e) verifies that the subset selected by Alg.~\ref{Alg1} usually has much smaller noise ratio than the original set. Sometimes, training DNNs requires larger number of training samples. Here we demonstrate that Alg.~\ref{Alg2} (INCV) can identify more clean samples through iteration. For efficiency, we use the ResNet-32 and set $N=4$, $E=50$ without fine tuning. $\varepsilon$ is estimated automatically using Eq.~(\ref{Eq_sym_noise}) in all experiments.

\begin{figure}[t]	
	%\vskip 0.2in
	\begin{center}
		\centerline{\includegraphics[width=0.9\columnwidth]{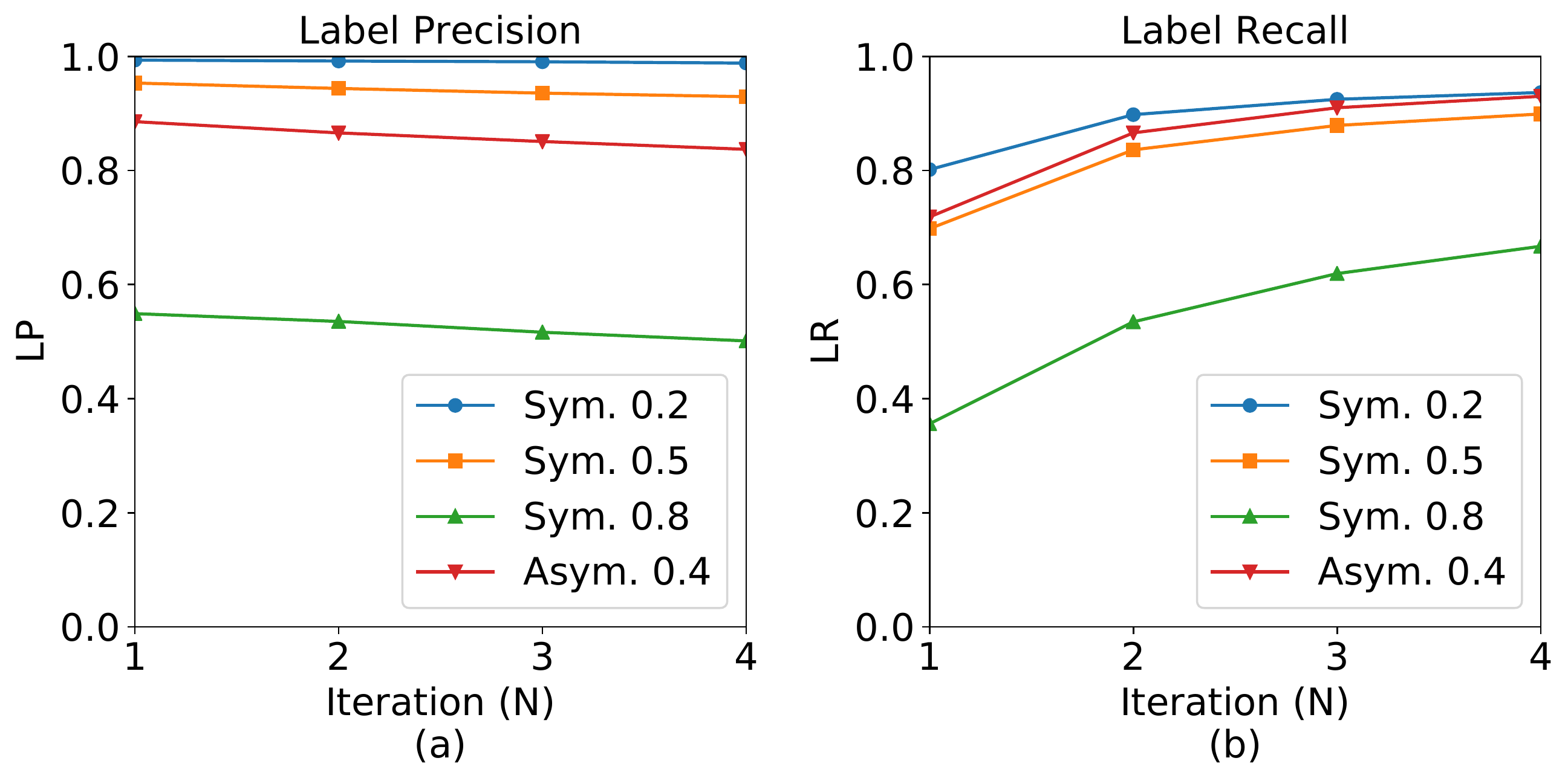}}
		\vskip -0.15in
		\caption{$LP$ and $LR$ of the INCV on the manually corrupted CIFAR-10.
			In each figure, the four curves correspond to symmetric noise of ratio $0.2$, $0.5$, $0.8$ and asymmetric noise of ratio $0.4$.
		}
		\label{Fig_Identify}
	\end{center}
	\vskip -0.4in
\end{figure}

\textbf{The INCV identifies most clean samples accurately.} Fig. ~\ref{Fig_Identify} illustrates the average $LP$ and $LR$ values of the Alg.~\ref{Alg2}, computed by repeating all experiments $5$ times. As show in the figure, the $LP$ and $LR$ are better than the theoretical lower bound even after a single iteration. Compared with ResNet-110 used in Sec.~\ref{Sec_exp_verify}, in this subsection we train the ResNet-32 for only $50$ epochs at each iteration. A much simpler model naturally releases the overfitting problem, yielding better $LP$ and $LR$. Besides, Fig.~\ref{Fig_Identify} also demonstrates that the $LR$ increases much with iteration, while the $LP$ slightly decreases. After four iterations, the INCV accurately identifies most clean samples. For example, under symmetric noise of ratio $0.5$, it selects about $90\%$ ($=LR$) of the clean samples, and the noise ratio of the selected set is reduced to around $10\%$ ($=1-LP$).

\textbf{Noisy labels exist even in the original CIFAR-10.} We also run the INCV on the original CIFAR-10 for just $1$ iteration and examine samples that are identified as corrupted ones. Interestingly, there are several confusing samples, as shown in Fig.~\ref{Fig_original}. This indicates that noisy labels exist even in the original CIFAR-10. Although corrupted samples contained in CIFAR-10 are so rare, which have negligible influence on training,  being capable of identifying them implies that the INCV is a powerful algorithm for cleaning noisy labels.

\subsection{Training DNNs robustly against noisy labels}
\label{Sec_exp_app}
As outlined in Alg.~\ref{Alg3}, we reformulate the Co-teaching to take full advantage of our INCV method. The followings clarify some questions that are useful for practical implementations of Alg.~\ref{Alg3}.
\begin{itemize}%[noitemsep]
	\item Q: \textit{How to set the size of mini-batches $\mathcal{B}_{\mathcal{C}}$ and $\mathcal{B}_{\mathcal{S}}$ drawn from $\mathcal{C}$ and $\mathcal{S}$?}\\
	A: In general, it is reasonable to draw mini-batches such that $\lvert\mathcal{B}_{\mathcal{C}}\rvert/\lvert\mathcal{B}_{\mathcal{S}}\rvert=\lvert\mathcal{C}\rvert/\lvert\mathcal{S}\rvert$. However, when $\mathcal{C}$ is large, it results in drawing too many samples from $\mathcal{C}$, which harms the training process since $\mathcal{C}$ usually contains many corrupted samples. Therefore, we adjust the strategy slightly by setting $\lvert\mathcal{B}_{\mathcal{C}}\rvert/\lvert\mathcal{B}_{\mathcal{S}}\rvert=\min(0.5,\lvert\mathcal{C}\rvert/\lvert\mathcal{S}\rvert)$. In the experiments, we set the batch size $\lvert\mathcal{B}_{\mathcal{S}}\rvert$ to $128$, then compute $\lvert\mathcal{B}_{\mathcal{C}}\rvert$ accordingly.
	
	\item Q: \textit{How many samples should we keep in each mini-batch?}\\
	A: In each mini-batch, we update the network using $\#n(e)$ samples that have small training loss, where $e$ is the current epoch. Following Co-teaching \cite{han2018co}, we set $n(e)=\lvert\mathcal{B}_{\mathcal{S}}\rvert(1-\varepsilon_{\mathcal{S}}\min(e/10,1))$, which means we decrease $n(e)$ from $\lvert\mathcal{B}_{\mathcal{S}}\rvert$ to $\lvert\mathcal{B}_{\mathcal{S}}\rvert(1-\varepsilon_{\mathcal{S}})$ linearly at the first $10$ epochs and fix it after that. Recall that $\varepsilon_{\mathcal{S}}=1-LP$ denotes the noise ratio of $\mathcal{S}$.
\end{itemize}

\begin{figure}[t]	
	\vskip -0.1in
	\begin{center}
		\centerline{\includegraphics[width=0.9\columnwidth]{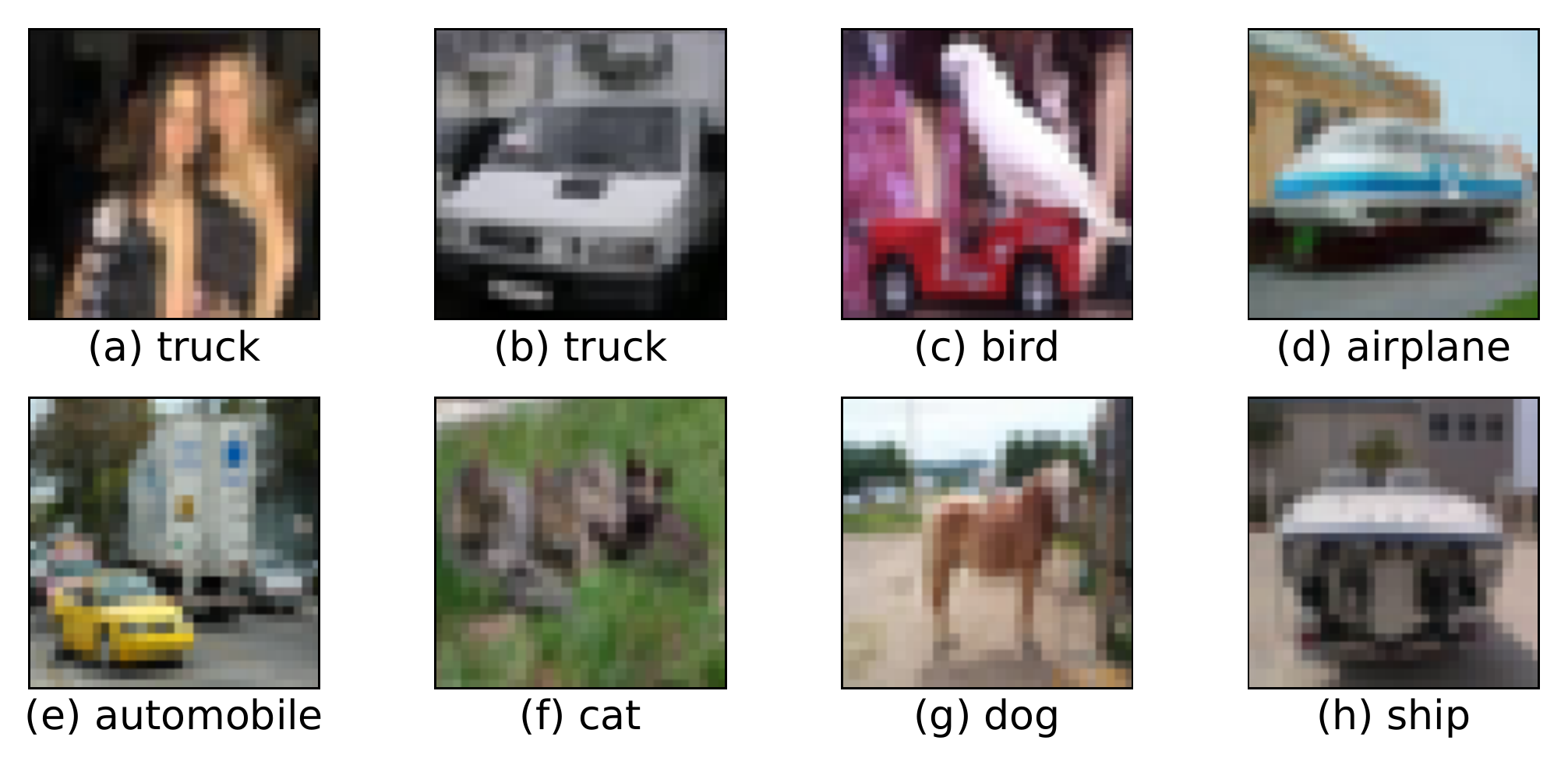}}
		\vskip -0.15in
		\caption{Noisy labels contained in the CIFAR-10 and identified by the INCV. Original Labels are annotated under images. (a) Human labeled as truck. (b) Labeled as truck, actually an automobile? (c) A bird on a toy car. (d) Labeled as airplane. (e) An automobile beside a truck. (f) Labeled as cat. (g) Labeled as dog, actually a horse? (h) Labeled as ship.}
		\label{Fig_original}
	\end{center}
	\vskip -0.4in
\end{figure}
%\textbf{Experimental details.} Still, we test the algorithm using the ResNet-32 and repeat each experiment five times. We first run the INCV to identify correct labels and estimate the noise ratio, then start the main training process as outlined in Alg.~\ref{Alg3}. We train DNNs on the selected training set $\mathcal{S}$ at the first $E_0$ epochs, where $E_0$ is set to $80$ for asymmetric noise or large symmetric noise and $40$ otherwise. For symmetric noise, we test on manually corrupted CIFAR-10 with noise ratio $0.2$, $0.5$ and $0.8$, and for asymmetric noise, we chose a non-trivial and challenge noise ratio $0.4$, since asymmetric noise larger than $0.5$ is trivial.

\textbf{Comparable methods.} We compare Alg.~\ref{Alg3} with the following baselines ($1$) \textit{F-correction} \cite{patrini2017making}. It first trains a network to estimate $T$, then corrects the loss function accordingly. ($2$) \textit{Decoupling} \cite{malach2017decoupling}. It trains two networks on samples for which the predictions from the two networks are different. ($3$) \textit{Co-teaching} \cite{han2018co}. It maintains two networks. Each network selects samples of small training loss from the mini-batches and feeds them to the other network. ($4$) \textit{MentorNet} \cite{jiang2018mentornet}. A teacher network is pre-trained, which provides a sample weighting scheme to train the student network. ($5$) \textit{D2L} \cite{ma2018dimensionality}. For each sample, it linearly combines the original label and the prediction of network as the new label. The combining weight depends on the dimensionality of the latent feature subspace \cite{amsaleg2017vulnerability}.

\textbf{Experiments on manually corrupted CIFAR-10.} We first evaluate all methods on the CIFAR-10 by manually corrupting the labels with different types of noise. For symmetric noise, we test noise ratio $0.2$, $0.5$ and $0.8$. For asymmetric noise, we choose a non-trivial and challenging noise ratio $0.4$, since asymmetric noise larger than $0.5$ is trivial. Still, we use the ResNet-32 and repeat all experiments five times. As shown in Table~\ref{Tab_cifar}, our method always achieves the best test accuracy (marked in boldface) under all cases. Even for symmetric noise of ratio $0.8$ which is challenging for most methods, we achieve a good test accuracy. Fig.~\ref{Fig_Test_Acc} illustrates the test accuracy of all methods on the clean test set after every training epoch. It can be found that our method impressively achieves the best test accuracy in all settings, while some baseline methods suffer from overfitting at the later stage of training, such as F-correction, Decoupling and MentorNet shown in Fig~\ref{Fig_Test_Acc} (b) $\&$ (d), and D2L shown in all four sub-figures. In particular, compared with the Co-teaching \cite{han2018co}, our method further enjoys a more stable training process and obtains better test accuracy by training on a clean subset firstly. %The performance of Co-teaching is extremely bad for symmetric noise of ratio $0.8$.

\begin{table}[t]
	\vskip -0.1in
	\centering
	\caption{Average test accuracy ($\%$, 5 runs) with standard deviation under different noise types and noise ratios. We train the RseNet-32 on manually corrupted CIFAR-10 and test on the clean test set. The best result is marked in bold face.\\}
	\begin{tabular}{|c|c|c|c|c|}
		\hline
		\multirow{2}{*}{Method} & \multicolumn{3}{c|}{Sym.} & Asym.       \\ \cline{2-5} 
		& $0.2$   & $0.5$   & $0.8$   & $0.4$                             \\ \hline
		
		\multirow{2}{*}{F-correction}
		& $85.08$ & $76.02$ & $34.76$ & $83.55$                           \\ 
		& \multicolumn{1}{l|}{$\pm0.43$} & \multicolumn{1}{l|}{$\pm0.19$}
		& \multicolumn{1}{l|}{$\pm4.53$} & \multicolumn{1}{l|}{$\pm2.15$} \\ \hline
		
		\multirow{2}{*}{Decoupling}
		& $86.72$ & $79.31$ & $36.90$ & $75.27$                           \\ 
		& \multicolumn{1}{l|}{$\pm0.32$} & \multicolumn{1}{l|}{$\pm0.62$}
		& \multicolumn{1}{l|}{$\pm4.61$} & \multicolumn{1}{l|}{$\pm0.83$} \\ \hline
		
		\multirow{2}{*}{Co-teaching}
		& $89.05$ & $82.12$ & $16.21$ & $84.55$                           \\ 
		& \multicolumn{1}{l|}{$\pm0.32$} & \multicolumn{1}{l|}{$\pm0.59$}
		& \multicolumn{1}{l|}{$\pm3.02$} & \multicolumn{1}{l|}{$\pm2.81$} \\ \hline
		
		\multirow{2}{*}{MentorNet}
		& $88.36$ & $77.10$ & $28.89$ & $77.33$                           \\ 
		& \multicolumn{1}{l|}{$\pm0.46$} & \multicolumn{1}{l|}{$\pm0.44$}
		& \multicolumn{1}{l|}{$\pm2.29$} & \multicolumn{1}{l|}{$\pm0.79$} \\ \hline
		
		\multirow{2}{*}{D2L}
		& $86.12$ & $67.39$ & $10.02$ & $85.57$                           \\ 
		& \multicolumn{1}{l|}{$\pm0.43$} & \multicolumn{1}{l|}{$\pm13.62$}
		& \multicolumn{1}{l|}{$\pm0.04$} & \multicolumn{1}{l|}{$\pm1.21$} \\ \hline
		
		\multirow{2}{*}{Ours}
		& $\textbf{89.71}$ & $\textbf{84.78}$ & $\textbf{52.27}$ & $\textbf{86.04}$ \\ 
		& \multicolumn{1}{l|}{$\pm0.18$} & \multicolumn{1}{l|}{$\pm0.33$}
		& \multicolumn{1}{l|}{$\pm3.50$} & \multicolumn{1}{l|}{$\pm0.54$} \\ \hline		
	\end{tabular}
	\label{Tab_cifar}
	\vskip -0.05in
\end{table}
\begin{figure}[t]	
	%\vskip 0.1in
	\begin{center}
		\centerline{\includegraphics[width=0.95\columnwidth]{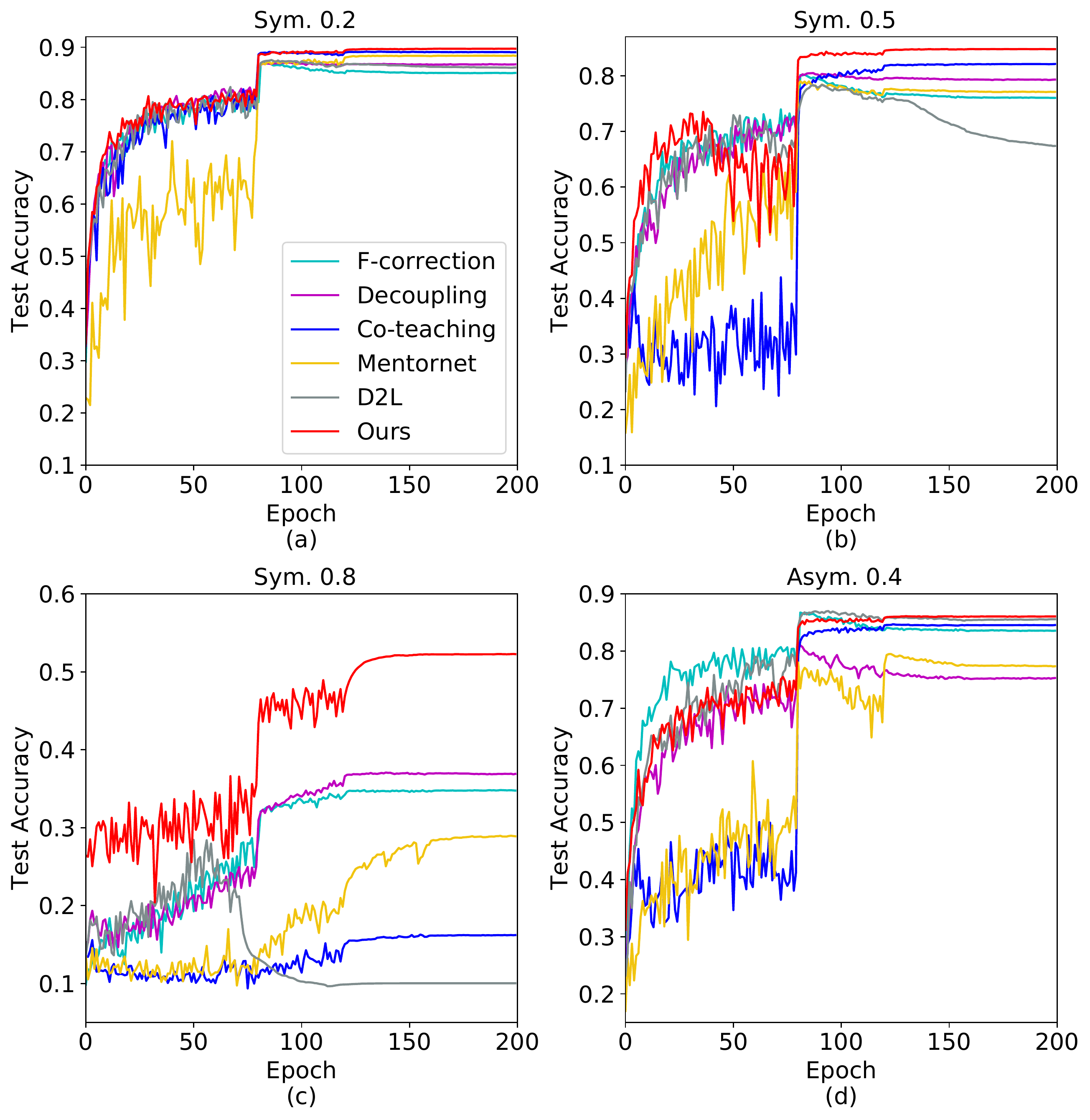}}
		\vskip -0.15in
		\caption{Average test accuracy (5 runs) %with standard deviation (shade)
			during training under different noise types and noise ratios. We train the RseNet-32 on manually corrupted CIFAR-10 and test on the clean test set. The sharp change of accuracy results from the learning rate change.}
		\label{Fig_Test_Acc}
	\end{center}
	\vskip -0.4in
\end{figure}

\begin{table}[t]
	\centering
	\caption{Validation accuracy ($\%$) on the WebVision validation set and ImageNet ILSVRC12 validation set. The number outside (inside) the parentheses
		denotes Top-1 (Top-5) classification accuracy. We train the inception-resnet v2 on the first 50 classes of the WebVision training set, which contains real-world noisy labels. The best result is marked in bold face.\\}
	\begin{tabular}{|c|c|c|}
		\hline
		Method       & WebVision Val. & ILSVRC2012 Val. \\ \hline
		F-correction & $61.12\,(82.68)$            & $57.36\,(82.36)$             \\
		Decoupling   & $62.54\,(84.74)$            & $58.26\,(82.26)$             \\
		Co-teaching  & $63.58\,(85.20)$            & $61.48\,(84.70)$             \\
		MentorNet    & $63.00\,(81.40)$            & $57.80\,(79.92)$             \\
		D2L          & $62.68\,(84.00)$            & $57.80\,(81.36)$             \\
		Ours         & $\textbf{65.24\,(85.34)}$            & $\textbf{61.60\,(84.98)}$             \\ \hline
	\end{tabular}
	\label{Tab_web}
	\vskip -0.1in
\end{table}

\textbf{Experiments on real-world noisy labels.} To verify the practical usage of our method on real-world noisy labels, we use the WebVision dataset $1.0$ \cite{li2017webvision}, whose training set contains many real-world noisy labels. Since the dataset is quite large, for quick experiments, we compare all methods on the first 50 classes of the Google image subset using the inception-resnet v2 \cite{szegedy2017inception}. We test the trained model on the human-annotated WebVision validation set and the ILSVRC12 validation set. As shown in table~\ref{Tab_web}, our method consistently outperforms other state-of-the-art ones in terms of test accuracy. Moreover, Supp.~C, contains some noisy examples identified automatically from the WebVision dataset by our INCV method (Alg.~\ref{Alg2}), which implies the INCV is reliable on datasets containing real-world noisy labels.

\section{Conclusion}
In this work, we initiate a formal study of noisy labels. We first formulate several findings towards the generalization of DNNs trained with noisy labels. Theoretical analysis and extensive experiments are presented to justify our statements. Based on our findings, we then propose the INCV method, which randomly divides noisy datasets, then utilizes cross-validation to identify clean samples. We provide theoretical guarantees for the INCV, and then demonstrate through experiments that it is capable of identifying most clean samples accurately. Finally, we adopt the Co-teaching strategy which takes full advantage of the identified samples to train DNNs robustly against noisy labels. By comparing with extensive baselines, we show that our method achieves state-of-the-art test accuracy on the clean test set. In future, our formulations on the generalization performance of DNNs trained with noisy labels may promote more fundamental approaches of dealing with label corruption.

%\section*{Acknowledgment}
\clearpage
\bibliography{example_paper}
\bibliographystyle{icml2019}

\twocolumn[
\icmltitle{Supplementary Materials: \\Understanding and Utilizing Deep Neural Networks\\ Trained with Noisy Labels}

\icmlkeywords{Machine Learning, ICML}

\vskip 0.3in
]

\begin{appendix}
	
	\begin{figure*}[t]	
		%\vskip 0.2in
		\begin{center}
			\centerline{\includegraphics[width=2\columnwidth]{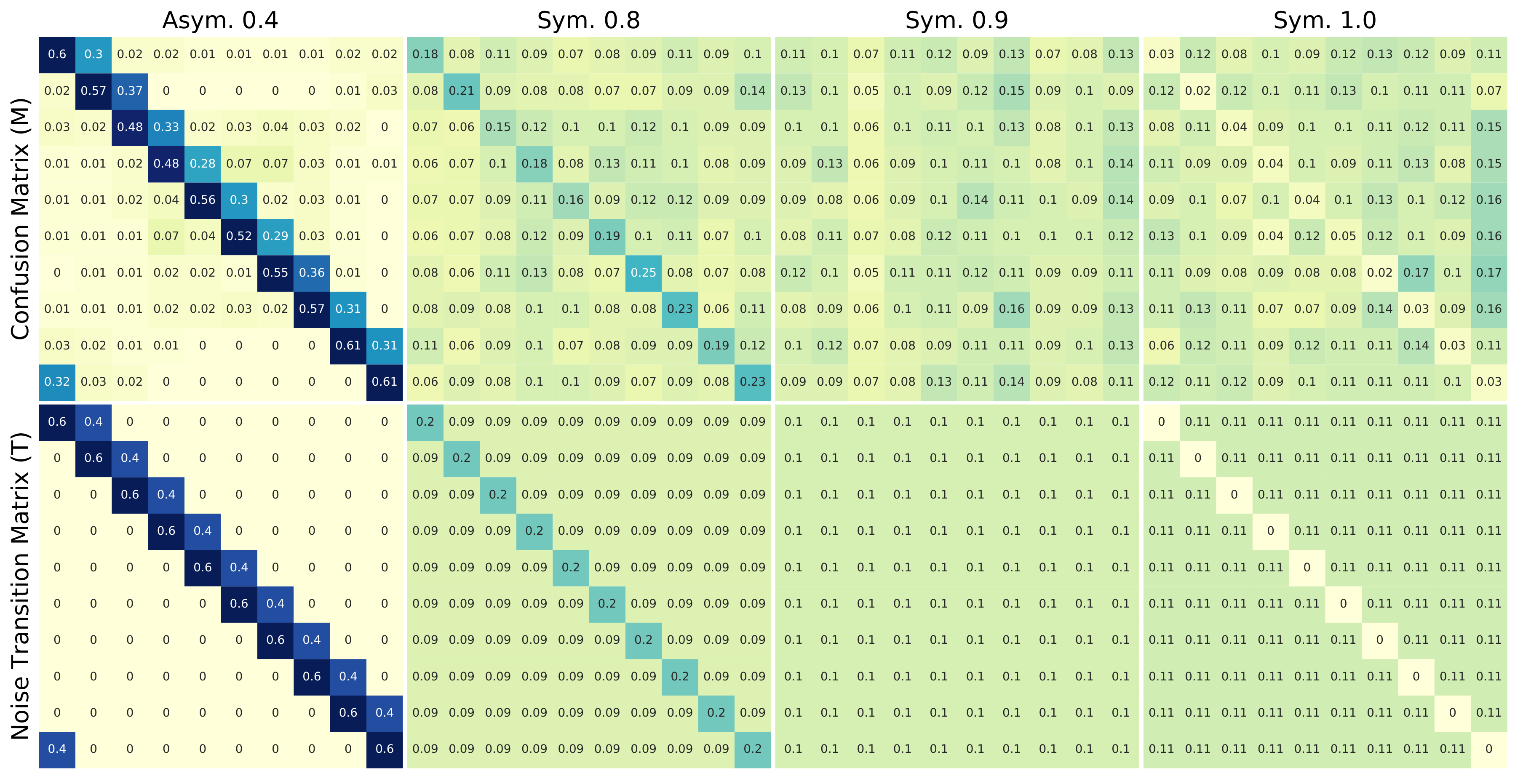}}
			\vskip -0.15in
			\caption{Confusion matrix (the first row) of ResNet-110 normally trained on corrupted CIFAR-10 with noise transition matrix $T$ (the second row). We specifically examine the noise settings with low training accuracy. $M\approx T$ satisfies the statement presented in Claim~1.}
			\label{Fig_Confusion2}
		\end{center}
		\vskip -0.2in
	\end{figure*}
	\begin{figure}	
		%\vskip 0.2in
		\begin{center}
			\centerline{\includegraphics[width=\columnwidth]{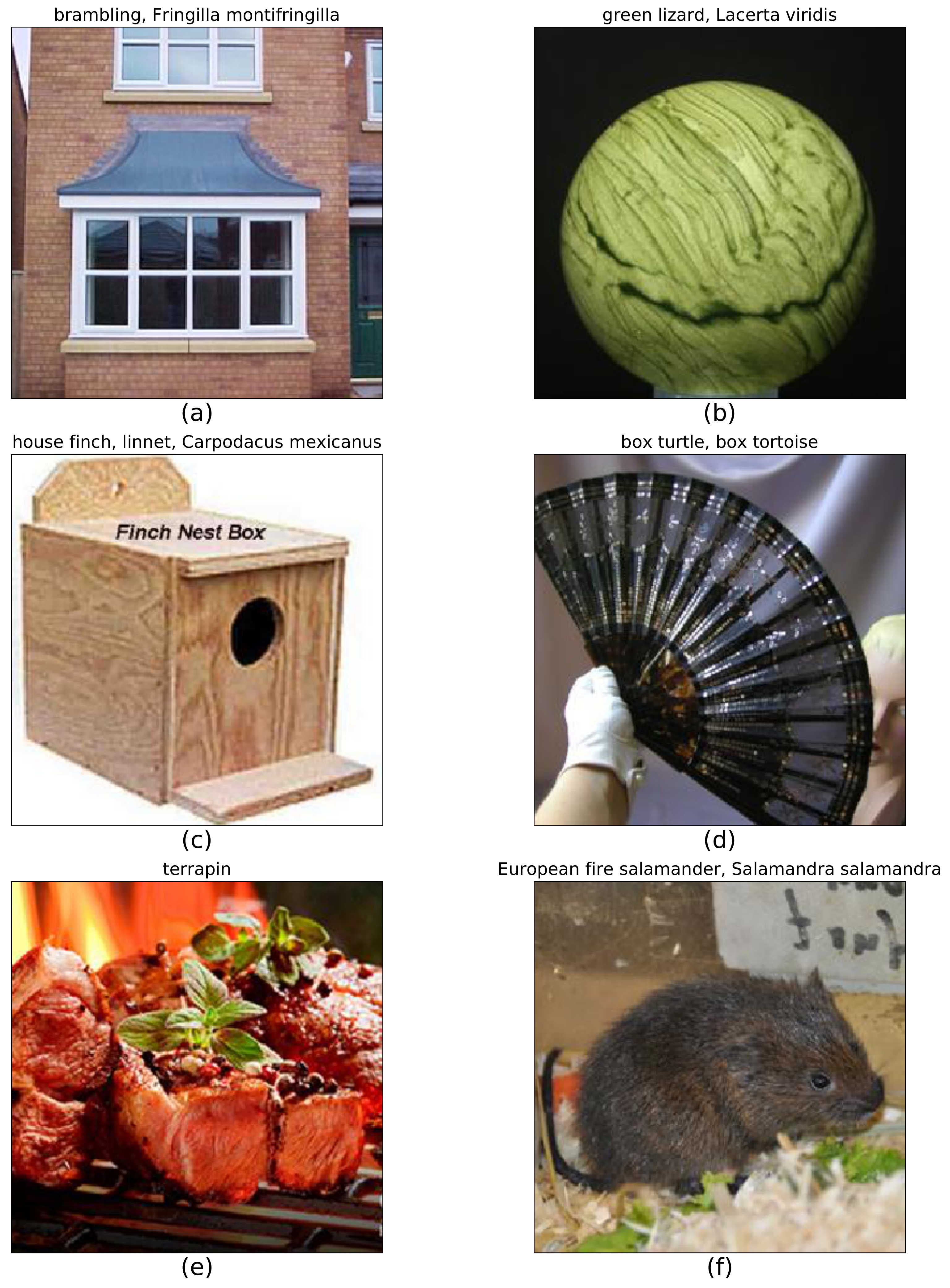}}
			\vskip -0.15in
			\caption{Examples of automatically identified noisy labels in the WebVision dataset using the INCV. We annotate the labeled concepts on top of each image. The labels are obviously unreasonable.}
			\label{Fig_web_noisy}
		\end{center}
		\vskip -0.2in
	\end{figure}
	\section{Further details on experiments}
	%In this Section, we outline more details of experiments conducted in the main paper. %Our code is contained in the submission.
	\label{Sup_exp}
	\subsection{CIFAR-10}
	CIFAR-10 \cite{krizhevsky2009learning} contains human-annotated labels which can be treated as true labels. To conduct experiments on the synthetic noisy labels, we randomly corrupt the labels according to a noise transition matrix $T$. %True labels are used to verify our theory and the sample-identification algorithm. We also compare generalization performance of baseline methods by training DNNs on the corrupted dataset and testing on the clean test set.
	
	In all experiments, we set the batch size to $128$, and implement (i) $l_2$ weight decay of $10^{-4}$ and (ii) data augmentation of horizontal random flipping and $32\times32$ random cropping after padding 4 pixels around images. In Sec.~5.1, we aim to verify our theory by demonstrating the worst case, so we use the ResNet-110 \cite{he2016identity} to ensure the model has the sufficiently high capacity to memorize all corrupted samples. While in Sec.~5.2 $~\&$ 5.3, we use the ResNet-32 \cite{he2016deep} for the consideration of training efficiency.
	
	In Sec.~5.2, we apply the Iterative Noisy Cross-Validation (INCV, Alg.~2) to select clean samples. For efficiency, we set the number of iterations to $4$, and train the ResNet32 for $50$ epochs at each iteration. We use the Adam optimizer with an initial learning rate $10^{-3}$, which is divided by $2$ after $20$ and $30$ epochs, and finally takes the value $10^{-4}$ after $40$ epochs. In all other experiments, we train the networks for $200$ epochs till convergence, using the Adam optimizer \cite{kinga2015method} with an initial learning rate $10^{-3}$, which is divided by $10$ after $80$, $120$ and $160$ epochs, and further divided by $2$ after $180$ epochs.
	
	After selecting clean samples, we train DNNs robustly using Alg.~3. We set the warm-up epochs $E_0$ to $40$ or $80$ (i.e., $20\%$ or $40\%$ of the total number of training epochs) without fine tuning. If the size of the candidate set $\mathcal{C}$ is large, considering it has much more noisy labels than the selected relatively clean set $\mathcal{S}$, we set $E_0=80$ so that the network will focus on $\mathcal{S}$ until $80$ epochs. Otherwise, we take $E_0=40$. In the INCV, we denote the proportion between the number of removed samples and selected samples as remove ratio $r$, which determines how many samples will be removed. We found that our algorithm is robust to $r$, which means slightly changing it does not affect the performance much. If we do not want to remove any samples, we set $r=0$, otherwise we set $r=\frac{\varepsilon}{1-\varepsilon}$ without fine tuning, where $\varepsilon$ is the estimated noise ratio of the original training set given by Alg.~2, hence $\frac{\varepsilon}{1-\varepsilon}$ is the proportion between the number of corrupted samples and the number of clean samples in the original training set.
	
	For those baseline methods, there are many specific hyperparameters, and we set the value according to their original papers. We train the same ResNet-32 for $200$ epochs using the Adam optimizer with the same learning rate scheduler.
	
	\subsection{WebVision}
	To verify the practical usage of our method on real-world noisy labels, we use the WebVision dataset $1.0$ \cite{li2017webvision} which contains 2.4 million images crawled from the websites using the 1,000 concepts in ImageNet ILSVRC12 \cite{deng2009imagenet}. The training set of the WebVision contains many real-world noisy labels. Since the dataset is quite large, for quick experiments, we use the first 50 classes of the Google image subset. We test the trained DNNs on the human-annotated WebVision validation set and the ILSVRC12 validation set.
	
	We use the inception-resnet v2 \cite{szegedy2017inception}. Following the standard training pipeline \cite{li2017webvision}, we first resize each image to make shorter size as $256$. Then we implement standard data augmentation: randomly crop a patch of size $227\times227$ form each image, and horizontal random flipping is applied before feeding the patch to the network for training. The batch size is set to $128$ for all experiments. We train the networks for $120$ epochs using the SGD optimizer with an initial learning rate $0.1$, which is divided by $10$ after $40$, and $80$ epochs.
	
	In our method, we first run the INCV to select clean samples. In the INCV, we set the number of iterations to $2$, and train the model for simply $50$ epochs at each iteration. We use the SGD optimizer with an initial learning rate $0.1$, which is divided by $2$ after $20$ and $30$ epochs, and finally takes the value $0.01$ after $40$ epochs. We set the remove ratio $r$ to $0.1$. After selecting clean samples, we train a model robustly using Alg.~3, where we set the warm-up epoch as $E_0=20$.
	
	\section{More plots of the confusion matrix}
	\label{Sup_Con}
	We have shown in the main paper that when a network is trained with noisy labels, its confusion matrix $M$ on the test set equals to the noise transition matrix $T$. This directly verifies our statement presented in Claim~1, which implies that the DNNs are able to fit the noisy training set exactly and generalize in distribution. Due to lack of space, in the main paper, we simply show results for symmetric noise of ratio $0.7$. Here in Fig.~\ref{Fig_Confusion2}, we show that $M\approx T$ holds for different noise types and noise ratios. We present the results for asymmetric noise of ratio $0.4$, and then specifically investigate the noise settings which result in a low training accuracy, i.e., symmetric noise of ratio $0.8$, $0.9$ and $1.0$ where the training accuracies are $0.40$, $0.24$ and $0.36$. In this way, we also verify that the training accuracy converging to a extremely low value does not contradict our formulations on the generalization performance of DNNs trained with noisy labels.
	
	\section{The INCV automatically identifies many noisy labels in the WebVision dataset}
	\label{Sup_noisy}
	In Sec.~5.3, we have demonstrated that on the WebVision dataset, compared with state-of-the-art methods, our training strategy is capable of training a model that achieves the best generalization performance on the clean validation set. In the experiments, we firstly select most clean samples out of the original training set use the Iterative Noisy Cross-Validation (INCV, Alg.~2). The INCV also identifies samples that are very likely to have a wrong label. In this Section, we demonstrate that the INCV does identify many noisy labels in the WebVision, as shown in Fig.~\ref{Fig_web_noisy}. Since the images have different size with shorter size as $256$, we crop each image from the center to form a square image. We first convert the observed label of each example to the correspond concept in the synsets, then annotate the concept on top of each image. In the WebVision, the $6$ images are labeled as (a) brambling, Fringilla montifringilla; (b) green lizard, Lacerta viridis; (c) house finch, linnet, Carpodacus mexicanus; (d) box turtle, box tortoise; (e) terrapin; (f) European fire salamander, Salamandra salamandra; which are obviously unreasonable.
	
	\section{More discussions on Corollary 2.2}
	Without loss of generality, we assume $\forall i$, $T_{ii}$ being the largest among $T_{ij}$, $j\in[c]:=\{1,\cdots,c\}$. Based on Corollary 2.2 presented in the main paper, we can prove that under the cases of symmetric and asymmetric noise, Alg.~1 always selects a subset with smaller noise ratio than the original dataset, i.e., $\varepsilon_S < \varepsilon$, where $\varepsilon$ is the noise ratio of the original dataset $\mathcal{D}$, and $\varepsilon_S$ is the noise ratio of the selected set $\mathcal{S}$. Recall that $\varepsilon_S=1-LP$ according to the definition of $LP$.
	
	For the symmetric noise, we have the definition $\forall i\in[c]$, $T_{ii}=1-\varepsilon$, and $T_{ij}=\varepsilon/(c-1), \forall j\neq i$. In this case, $T_{ii}$ being the largest number among $T_{ij}$ implies $\varepsilon/(c-1)<1-\varepsilon$. Using Eq.~(10) in Corollary 2.2, we have
	\begin{equation}
	\begin{aligned}
	\nonumber
	1-\varepsilon_S=LP&=\frac{(1-\varepsilon)^{2}}{(1-\varepsilon)^{2}+\varepsilon^2/(c-1)}\\
	&>\frac{(1-\varepsilon)^{2}}{(1-\varepsilon)^{2}+\varepsilon(1-\varepsilon)}\\
	&=1-\varepsilon.		
	\end{aligned}
	\end{equation}
	For the asymmetric noise, we have the definition $\forall i\in[c]$, $T_{ii}=1-\varepsilon$, $T_{ij}=\varepsilon$ for some $j\neq i$, and $T_{ij}=0$ otherwise. In this case, $T_{ii}$ being the largest number among $T_{ij}$ implies $\varepsilon<1-\varepsilon$. Using Eq.~(11) in Corollary 2.2, we have
	\begin{equation}
	\begin{aligned}
	\nonumber
	1-\varepsilon_S=LP&=\frac{(1-\varepsilon)^{2}}{(1-\varepsilon)^{2}+\varepsilon^2}\\
	&>\frac{(1-\varepsilon)^{2}}{(1-\varepsilon)^{2}+\varepsilon(1-\varepsilon)}\\
	&=1-\varepsilon.		
	\end{aligned}
	\end{equation}
	Thus, we can conclude that $\varepsilon_S<\varepsilon$.
	%For the asymmetric noise of  ratio $\varepsilon$, we have
	%\begin{equation}
	%\label{EQ_LR}
	%\begin{aligned}
	%LP=\frac{(1-\varepsilon)^{2}}{(1-\varepsilon)^{2}+\varepsilon^2},\quad LR=1-\varepsilon.
	%\end{aligned}
	%\end{equation}
\end{appendix}

\end{document}